\pdftrailerid{}
\documentclass[11pt]{article}
\usepackage[letterpaper,margin=1in]{geometry}
\usepackage[T1]{fontenc}
\usepackage{lmodern}
\usepackage{amsmath,amssymb,amsthm}
\usepackage{graphicx,booktabs,microtype}
\usepackage{float}
\usepackage[round]{natbib}
\setcitestyle{citesep={,}}
\usepackage[hidelinks]{hyperref}
\usepackage{url}

\newtheorem{theorem}{Theorem}
\newtheorem{lemma}[theorem]{Lemma}
\newtheorem{proposition}[theorem]{Proposition}
\newtheorem{corollary}[theorem]{Corollary}
\theoremstyle{remark}

\DeclareMathOperator{\conv}{conv}

\DeclareMathOperator*{\argmin}{arg\,min}
\newcommand{\R}{\mathbb R}
\newcommand{\ip}[2]{\langle #1,#2\rangle}
\newcommand{\norm}[1]{\lVert #1\rVert_2}
\hypersetup{pdftitle={Online Frank–Wolfe Cannot Beat T\textasciicircum{3/4}: Lower Bounds for Linear-Oracle Online Learning},pdfauthor={Mohit Sinha},pdfcreator={},pdfproducer={}}
\title{Online Frank--Wolfe Cannot Beat \(T^{3/4}\):\\
Lower Bounds for Linear-Oracle Online Learning}
\author{Mohit Sinha}
\date{}
\begin{document}
\maketitle
\begin{abstract}
Can a constant number of linear minimizations per round improve on the \(T^{3/4}\) regret rate of online Frank--Wolfe on general convex sets? Weibel et al.\ conjectured that fixed-coefficient methods cannot. We prove their conjecture and extend the lower bound to every deterministic learner in an oracle-only model. The learner receives an initial feasible point and a diameter bound, and must remain feasible on every domain consistent with its oracle replies. For \(T\) rounds, at most \(b\) calls between decisions, diameter bound \(D\), and gradient norm bound \(L\), we construct an instance in dimension \(d=2b(T-1)+1\) with regret at least \(2^{-1/4}LD\,b^{-1/4}T^{3/4}\). The adversary fixes the domain, initial point, deterministic tie rule and linear losses before play. The vertices form a path on which every point available before a decision has zero current loss, while the final vertex has negative loss on every round. For constant \(b\), the result matches the known upper rate for dimension-independent guarantees. For one-call fixed schedules with a nonzero coefficient on the newest gradient, a second construction gives regret at least \(3LD T^{3/4}/4\) with unique minimizers at every issued query. Exact-arithmetic certificates for the tuned schedule of Weibel et al.\ closely match their finite-horizon numerical worst cases, with unique oracle replies.
\end{abstract}

\raggedbottom
\section{Introduction}
Online Frank--Wolfe replaces projection onto a convex set by linear minimization over it. This can reduce the cost of a round substantially, as linear minimization over a nuclear-norm ball requires only a leading singular-vector pair \citep{Jaggi2013}. With one linear minimization oracle (LMO) call per round, the method of \citet[Algorithm~1 and Theorem~4.4]{HazanKale2012} achieves regret of order \(T^{3/4}\) over \(T\) rounds. Whether a constant number of calls permits a dimension-independent improvement on general convex sets is a central question in projection-free online learning \citep{HazanMinasyan2020,Mhammedi2025}. Following a numerical performance-estimation study, \citet[Section~4]{WeibelEtAl2026} conjectured a matching lower bound for methods with coefficients fixed in advance. We prove their conjecture.

We prove a matching lower bound for every deterministic learner in an oracle-only model. The learner receives an initial feasible point and a diameter upper bound, accesses the domain through an exact LMO, and must remain feasible for every domain consistent with that information. Its queries and computation may otherwise depend arbitrarily on the observed history. With at most \(b\) calls between consecutive decisions, some fixed domain, oracle and linear loss sequence force regret at least
\[
2^{-1/4}LD\,b^{-1/4}T^{3/4},
\]
where \(L\) bounds gradient norms and \(D\) bounds the diameter. The required dimension is \(2b(T-1)+1\). Thus constant-call deterministic methods have minimax exponent \(3/4\) when the guarantee must hold in every dimension.

The construction arranges the vertices along a path. At every round, all points the learner can play have loss zero, while a common unreached vertex has negative loss. Past gradients and returned vertices give every sufficiently late vertex the same query value, so a fixed least-index oracle keeps the learner behind the next loss. For fixed \(b\), the geometry of the path makes each loss difference of order \(LD T^{-1/4}\). A resisting rotation extends this argument from queries in the span of observed vectors to arbitrary deterministic queries.

The path construction uses a chosen oracle tie rule. For one-call fixed schedules with a nonzero coefficient on the newest gradient, we further construct instances with unique minimizers at every issued query, so every exact LMO gives the same hard execution. These instances have regret at least \(3LD T^{3/4}/4\) in dimension \(2T+1\). For the tuned online Frank--Wolfe schedule, exact-arithmetic certificates give regret at least \(1.18LD T^{3/4}\) at \(T=10,20,30,40,50,60\) (Table~\ref{tab:certified-examples}).

\paragraph{Related work.}
Domain geometry and loss structure permit faster rates. On \(d\)-dimensional polytopes, \citet{GarberHazan2016} obtain \(O(\sqrt{dT})\) regret with polytope-dependent constants. Strongly convex domains \citep{WanZhang2021} or strongly convex losses \citep{KretzuGarber2021} permit \(O(T^{2/3})\). On strongly convex domains, \citet[Theorem~8]{Mhammedi2022Curvature} obtains \(\widetilde O(\sqrt T)\) with at most two linear minimizations per round, where \(\widetilde O\) suppresses logarithmic factors and the constants depend on the geometry. Randomized guarantees can depend on dimension \citep{HazanMinasyan2020}. Our hard losses are linear, hence have constant gradients and smoothness zero. The construction uses domains whose dimension grows with the horizon.

Membership \citep{Mhammedi2022Membership} and separation \citep{Mhammedi2025} oracles provide different information about the domain. For adaptive regret, which controls regret on every subinterval, \citet{GarberKretzu2022Adaptive} obtain \(O(T^{3/4})\) with \(O(T)\) total linear-oracle calls. Our budget limits the calls between each pair of decisions, and our feasibility requirement is stated explicitly below.

Resisting-oracle arguments are classical \citep{NemirovskiYudin1983}. The fixed-schedule specialization is also related to oblivious first-order methods \citep{ArjevaniShamir2016}. Performance estimation \citep{DroriTeboulle2014} and interpolation \citep{TaylorHendrickxGlineur2017Interpolation} describe finite worst-case problems. Related oracle-complexity results concern smooth convex minimization \citep{Drori2017Exact} and the distance-to-solution criterion for smooth strongly convex minimization \citep{DroriTaylor2022Oracle}. We use the online performance-estimation problem (PEP) of \citet{WeibelEtAl2026} for sharper schedule-specific information, retaining the harmonic analysis in the appendix.

The direct online-to-batch use of the nonsmooth oracle lower bound of \citet[Theorem~2]{Lan2013} yields only the \(\sqrt T\) scale, so our construction uses changing losses to keep a single comparator ahead of the available points. Recent offline results concern strongly convex sets \citep{GrimmerLiu2026,HalbeyEtAl2026StronglyConvex} and curvature-dependent Frank--Wolfe rates \citep{HalbeyRouxPokutta2026Curvature}.

\section{Model}\label{sec-model}
Let \(T,b\ge1\) be integers, let \(d\) be the ambient dimension, and let \(L,D>0\). We use the Euclidean inner product \(\ip{\cdot}{\cdot}\) and norm \(\norm{\cdot}\). The unknown domain \(\mathcal C\subset\R^d\) is nonempty, compact and convex, with diameter at most \(D\). The learner receives \(T,b,d,L,D\) and a point \(x_1\in\mathcal C\), which is its first decision. On round \(t\), it plays \(x_t\), observes a gradient \(g_t\) with \(\norm{g_t}\le L\), and incurs loss \(\ip{g_t}{x_t}\). Regret is
\begin{equation}\label{eq-regret}
R_T=\sum_{t=1}^T\ip{g_t}{x_t}
-\min_{u\in\mathcal C}\sum_{t=1}^T\ip{g_t}{u}.
\end{equation}
An exact LMO is a fixed map \(\mathcal O:\R^d\to\mathcal C\) with
\[
\mathcal O(q)\in\argmin_{v\in\mathcal C}\ip{q}{v}.
\]
The learner has no further domain information or access. After observing \(g_t\), it may make at most \(b\) sequential oracle calls before playing \(x_{t+1}\), and later queries in that round may use earlier replies from the same round. There are no calls before \(x_1\), and calls after loss \(T\) have no effect on regret. The budget applies separately between each pair of decisions.

A deterministic learner in this model must produce feasible decisions for every admissible domain, oracle and loss sequence. In particular, \(D\) is an announced upper bound, even when a constructed domain attains it. The points already seen could themselves generate the entire domain, so feasibility forces the learner to stay in their convex hull. Write \(\conv\) for convex hull, and let \(\mathcal H_t\) be the convex hull of \(x_1\) and the replies received before round \(t\).

\begin{lemma}[Feasibility from the observed points]\label{lem-hull}
Every deterministic learner in this model plays \(x_t\in\mathcal H_t\).
\end{lemma}
\begin{proof}
The domain \(\mathcal H_t\) is compact, convex and has diameter at most \(D\). Every past reply minimizes its query over \(\mathcal H_t\), since all its points lie in the original domain. Repeated queries already have consistent replies. Extend those replies to an exact fixed oracle on \(\mathcal H_t\) by selecting a minimizer elsewhere. The same gradients and this oracle reproduce the history. Feasibility on that admissible domain therefore requires \(x_t\in\mathcal H_t\).
\end{proof}

\section{The path construction}\label{sec-chain}
The path must fit inside a bounded domain while keeping all sufficiently late vertices indistinguishable to the learner's queries. We arrange the increments so that each returned vertex has the same inner product with every later vertex. Each loss then separates a prefix of the path from its remaining vertices, rewarding only those beyond the learner's reach.

We first give one instance for every learner whose queries lie in the linear span of observed gradients and returned points and whose decisions lie in their observed convex hull. Coefficients in this intermediate class may be adaptive or randomized. Set \(x_1=0\), and let \(M=b(T-1)+1\). A Gram matrix records the pairwise inner products of a list of vectors. To specify the path increments, define the tridiagonal matrix \(\mathbf A_{\rm p}\in\R^{M\times M}\) to have diagonal entries \(2\), adjacent entries \(-1\), and all other entries zero. Put
\[
\mathbf K=\frac{D^2}{4}\left(\mathbf A_{\rm p}+\frac2M I_M\right),
\]
where \(I_M\) is the identity matrix. Choose linearly independent increments \(\Delta_1,\ldots,\Delta_M\in\R^M\) with Gram matrix \(\mathbf K\), and define \(M+1\) vertices by
\[
w_1=0,\qquad w_j=-\sum_{i<j}\Delta_i\quad(2\le j\le M+1).
\]
Their convex hull is the domain. The fixed oracle selects the least-index minimizing \(w_j\).

Three relations give the required diameter, equal query scores on later vertices, and bounded gradient norms. For \(1\le i<j\le M+1\),
\begin{equation}\label{eq-path-distance}
\norm{w_i-w_j}^2=\frac{D^2}{4}\left(2+\frac{2(j-i)}M\right),
\end{equation}
so the diameter is exactly \(D\). For \(2\le i<j\le M+1\),
\begin{equation}\label{eq-path-prefix}
\ip{w_i}{w_j}=\frac{D^2}{4}\left(1+\frac{2(i-1)}M\right),
\end{equation}
so an observed vertex gives the same query contribution to every later vertex. To separate a prefix from the remaining vertices, let \(\Delta_r^*\in\R^M\) be the dual increments, defined by \(\ip{\Delta_r^*}{\Delta_i}=1\) when \(r=i\) and zero otherwise. Their norms satisfy
\begin{equation}\label{eq-path-inverse}
\norm{\Delta_r^*}^2=(\mathbf K^{-1})_{rr}
\le\frac{\sqrt{2M}}{D^2}.
\end{equation}
The first two identities follow by summing contiguous entries of the tridiagonal Gram matrix. Appendix~\ref{app-chain} proves all three, including the inverse bound.

For \(1\le t\le T\), define \(k_t=1+b(t-1)\), which will bound the returned vertex indices before round \(t\). The dual increment at \(k_t\) distinguishes the first \(k_t\) vertices from the rest. Set the loss gap to \(c_{\rm p}=LD(2M)^{-1/4}\) and choose the fixed gradients
\[
g_t=c_{\rm p}\Delta_{k_t}^*.
\]
Equation~\eqref{eq-path-inverse} gives \(\norm{g_t}\le L\), while
\begin{equation}\label{eq-path-scores}
\ip{g_t}{w_j}=\begin{cases}0,&j\le k_t,\\-c_{\rm p},&j>k_t.\end{cases}
\end{equation}

\begin{proposition}[One instance for span queries]\label{prop-span}
On this domain, oracle and loss sequence, every learner with the query-span and observed-hull restrictions above has \(R_T=c_{\rm p}T\) on every run.
\end{proposition}
\begin{proof}
Let \(K\) be the largest vertex index returned so far, including the initial index \(1\). After \(g_t\) is revealed, every available gradient assigns the same loss to all vertices with index greater than \(k_t\), by \eqref{eq-path-scores}. The inner product with any returned point is constant over later vertices, by \eqref{eq-path-prefix}. Thus a query in their span has one common score on all vertices with index greater than \(\max\{K,k_t\}\). The least-index rule implies
\[
K_{\rm new}\le\max\{K,k_t\}+1.
\]
Initially \(K=1=k_1\), so for \(t<T\), at most \(b\) calls in round \(t\) leave \(K\le k_t+b=k_{t+1}\). Before each decision, all available vertices have index at most \(k_t\), so \(\ip{g_t}{x_t}=0\). The last vertex \(w_{M+1}\) attains the minimum \(-c_{\rm p}\) on every round. Regret equals \(c_{\rm p}T\). This induction holds for every realized choice of coefficients.
\end{proof}

The same instance defeats every randomized learner satisfying the span-query and observed-hull restrictions, for every realization of its randomness. Thus randomizing the coefficients alone cannot improve the rate. Randomized queries outside the observed span remain beyond this argument. The dimension-dependent randomized guarantees of \citet{HazanMinasyan2020} motivate this possibility, while our results leave the dimension-independent question unresolved.

\subsection{Arbitrary deterministic queries}
An arbitrary query may contain directions outside the span of the observed vectors. We place the unseen part of the path orthogonally to each such new direction, preserving everything already revealed. The query then has the same scores on the domain as its component in the observed span. One extra dimension per call makes these choices compatible, and determinism lets the completed instance be fixed before play.
\begin{theorem}[Deterministic oracle lower bound]\label{thm-main}
Fix \(T,b\ge1\), \(L,D>0\), and \(d=2b(T-1)+1\). For every deterministic learner in the stated model in \(\R^d\), there exist a domain of diameter exactly \(D\), an initial point, a fixed deterministic vertex-valued exact LMO and a fixed sequence of gradients of norm at most \(L\), all chosen before play, such that, with \(M=b(T-1)+1\),
\begin{equation}\label{eq-main-lower}
R_T\ge LD\,T(2M)^{-1/4}
\ge 2^{-1/4}LD\,b^{-1/4}T^{3/4}.
\end{equation}
The oracle may be chosen to select the least-index minimizing vertex in one fixed ordering.
\end{theorem}
\begin{proof}[Proof of Theorem~\ref{thm-main}]
There are at most \(n_{\rm call}=b(T-1)\) calls. Embed the abstract path space \(\R^M\) isometrically in \(\R^{M+n_{\rm call}}\), exposing gradients and replies as the interaction proceeds. Maintain the images of every exposed vector, and a subspace \(\mathcal N\) orthogonal to the entire embedded path space. At a new query \(q\), let \(\mathcal S\) be the span of the exposed vectors and write
\[
q=q_{\mathcal S}+q_{\mathcal N}+r,
\qquad r\perp(\mathcal S+\mathcal N).
\]
If \(r\ne0\), move the unexposed part of the embedding to be orthogonal to \(r\), fixing \(\mathcal S\) pointwise, and add \(r\) to \(\mathcal N\). Before a call, fewer than \(n_{\rm call}\) independent normal directions have been stored, so there is room to retain the full \(M\)-dimensional path space. Every vertex then has the same query score as under \(q_{\mathcal S}\), whose pullback is a span query. Proposition~\ref{prop-span}'s index induction applies, and Lemma~\ref{lem-hull} puts every decision in the observed hull.

After this deterministic interaction, freeze the final embedding, all gradients and the least-index oracle. During the construction, each rotation preserves every previously exposed vector and keeps all stored normals orthogonal to the path space, so the frozen instance reproduces every query reply. Deterministic replay therefore reproduces the whole interaction, with regret \(c_{\rm p}T\). Appendix~\ref{app-rotation} supplies the dimension and replay details. For \(T=1\), the path has one increment and there are no calls, and the same loss calculation applies. Finally \(M\le bT\) proves the second inequality.
\end{proof}

The selected oracle can break ties against the learner, but remains unchanged throughout play. The domain and loss sequence in Theorem~\ref{thm-main} may depend on the deterministic learner. The theorem covers adaptive queries and decision rules whenever they satisfy the information and feasibility requirements above. It makes no claim for arbitrary-query randomized learners.

For \(b=1\) and \(T\ge3\), the upper bound \(4LD T^{3/4}/3^{3/4}\) of \citet[Theorem~3.1]{WeibelEtAl2026} gives a matching exponent with coefficient approximately \(1.755\). The upper-to-lower constant ratio is approximately \(2.09\). For any fixed \(b\), the one-call method is still available, so the exponent remains optimal. At \(b=T/4\), for horizons divisible by four, the second bound in \eqref{eq-main-lower} is \(2^{1/4}LD\sqrt T\), consistent with the square-root rate. The optimal dependence on a growing call budget remains open.

\section{Unique minimizers for fixed schedules}\label{sec-strict}
The chain gives all sufficiently late vertices the same query value. For fixed schedules, we remove the dependence on the oracle's tie rule by combining its Gram matrix with one whose prescribed replies strictly beat every competitor. Oracle comparisons are linear in Gram entries, so a small positive weight on the second matrix makes every issued minimizer unique while retaining most of the chain's regret.

\begin{samepage}
We now specialize to one call per round. Fix arrays \(\eta,\beta,\gamma\) before the instance is selected, and for \(1\le t<T\) set
\begin{equation}\label{eq-model}
\begin{aligned}
q_t&=\sum_{s\le t}\eta_{ts}g_s+\sum_{s<t}\beta_{ts}(v_s-x_1),\\
v_t&=\mathcal O(q_t),\qquad
x_t=x_1+\sum_{s<t}\gamma_{ts}(v_s-x_1).
\end{aligned}
\end{equation}
\end{samepage}
The decision formula also applies at \(t=T\). The coefficients may depend on \(T,L,D,d\). We require \(\gamma_{ts}\ge0\) and \(\sum_{s<t}\gamma_{ts}\le1\). This is the fixed-coefficient class of \citet{WeibelEtAl2026}, which includes the online conditional-gradient algorithm with predetermined step sizes in \citet[Section~7.5, Algorithm~24]{Hazan2016OCO}. The multiple-call model of \citet{WeibelEtAl2026} also permits earlier replies from the same round in later queries, and is a specialization of our timing convention.

The strict construction needs \(\eta_{tt}\ne0\). This lets the current gradient supply a new direction along which the next reply improves its query value. Moving the earlier part of that reply slightly toward the origin also preserves strictness at the previous queries.

\begin{corollary}\label{cor-strict}
Fix \(T\ge1\), \(L,D>0\), and arrays satisfying \eqref{eq-model} in dimension \(d=2T+1\), with \(\eta_{tt}\ne0\) for \(1\le t<T\). There is a domain of diameter exactly \(D\), an initial point and a loss sequence fixed before play such that every exact LMO yields
\[
R_T\ge\frac34LD\,T^{3/4}.
\]
At every issued query, the minimizing point is unique.
\end{corollary}
\begin{proof}
For a symmetric matrix \(W\), write \(W\succeq0\) for positive semidefiniteness and \(W\succ0\) for positive definiteness. Normalize \(L=D=1\) for the moment. Run the schedule on the path instance and let \(\mathbf G_{\rm ch}\) be the Gram matrix of its \(T\) gradients, \(T-1\) replies and comparator. Each squared norm, squared distance and oracle comparison is linear in these Gram entries, as is the regret against the comparator. Appendix~\ref{app-strict} constructs a second feasible Gram matrix \(\mathbf G_{\rm str}\) with zero objective and strict oracle comparisons against every other point.

The mixture \(\mathbf G_{\rm mix}=\tfrac9{10}\mathbf G_{\rm ch}+\tfrac1{10}\mathbf G_{\rm str}\) is feasible, preserves strict comparisons, and has objective \(\tfrac9{10}2^{-1/4}T^{3/4}>\tfrac34T^{3/4}\). Factoring this \(2T\times2T\) matrix realizes the vectors. Each prescribed reply is the unique minimizer over the convex hull of zero, the replies and the comparator, so every exact selector reproduces the transcript. Appendix~\ref{app-strict} verifies diameter padding in one further coordinate and restores \(L,D\). For \(T=1\), a segment with one loss gives regret \(LD\).
\end{proof}

Uniqueness concerns the directions issued in this execution. The corollary allows either sign of \(\eta_{tt}\), but a zero query on a nonsingleton domain cannot have a unique minimizer. Table~\ref{tab-scope} separates two kinds of guarantee. Theorem~\ref{thm-main} allows arbitrary adaptive deterministic queries with a chosen tie rule, while Corollary~\ref{cor-strict} removes the tie-rule dependence for its fixed schedules.
\begin{table}[t]
\centering\small\setlength{\tabcolsep}{4pt}
\begin{tabular}{@{}lcc@{}}
\toprule
& Theorem~\ref{thm-main} & Corollary~\ref{cor-strict}\\
\midrule
Learner & Deterministic & Fixed coefficients\\
Queries & Arbitrary & \(\eta_{tt}\ne0\)\\
Calls per round & \(b\) & \(1\)\\
Oracle & Selected fixed rule & Every exact rule\\
Dimension & \(2b(T-1)+1\) & \(2T+1\)\\
Coefficient, \(b=1\) & \(2^{-1/4}\) & \(3/4\)\\
\bottomrule
\end{tabular}
\caption{Lower bounds in the stated oracle-only model. The last row gives the coefficient of \(LD T^{3/4}\) for one call per round.}\label{tab-scope}
\end{table}

\paragraph{The tuned online Frank--Wolfe schedule.}
For \(T\ge3\), define
\[
\theta=\frac{3^{3/4}D}{2LT^{3/4}},\qquad
\sigma=\min\{1,\sqrt{3/T}\}.
\]
The method of \citet[Algorithm~1]{WeibelEtAl2026} uses
\[
q_t=\theta\sum_{s\le t}g_s+(x_t-x_1),\quad
x_{t+1}=(1-\sigma)x_t+\sigma v_t.
\]
It has \(\eta_{ts}=\theta\) for \(s\le t\) and
\(\beta_{ts}=\gamma_{ts}=\sigma(1-\sigma)^{t-1-s}\) for \(s<t\), with nonnegative row sums at most one. In particular \(\eta_{tt}=\theta>0\), so Corollary~\ref{cor-strict} applies. Its lower constant \(3/4\) and the proved upper constant \(4/3^{3/4}\) differ by a factor of approximately \(2.34\).

\section{Schedule-specific bounds and examples}\label{sec-bounds}
The common path instance gives the same lower bound for every schedule, but the decision weights can reveal larger regret. If a decision leans on a few outputs, a simplex construction makes its loss large. Spreading the weight reduces this concentration, but leaves more weight on old outputs that changing losses can penalize. Performance estimation quantifies these two effects.

For a one-call fixed schedule in \eqref{eq-model}, performance estimation expresses a worst-case transcript through its Gram matrix. With \(L=D=1\) and \(x_1=0\), let \(u\) be a comparator and let the point list contain \(0\), all oracle replies and \(u\). The semidefinite program (SDP) maximizes \(\sum_t\ip{g_t}{x_t-u}\) subject to positive semidefiniteness, gradient norm bounds, all pairwise diameter bounds and every oracle comparison. We use the full fixed-schedule PEP of \citet[Section~2.1]{WeibelEtAl2026}. Their design relaxation in Section~2.2 omits some comparisons. Appendix~\ref{app-canonicalization} gives the full formulation, and Appendix~\ref{app-fixed-oracle} realizes its value arbitrarily closely with a fixed selector.

The PEP also yields bounds tailored to the decision weights. A formal query is its coefficient vector over gradient and output symbols before choosing an instance. For this lower-bound argument, choose an oracle invariant under positive rescaling. Merge queries that are positive multiples of an earlier formal query, retain at most one zero query, and substitute the corresponding output aliases. Write \(m\) for the number of retained calls, \(\tau_r\) for the round of call \(r\), and \(\Gamma_{tr}\) for its weight in decision \(t\). For \(1\le k\le t\le T\), define
\[
\begin{aligned}
a_t&=\left(\sum_{r=1}^m\Gamma_{tr}^2\right)^{1/2},\qquad A=\sum_{t=1}^T a_t,\\
h_{t,k}&=1-\sum_{r:k\le\tau_r<t}\Gamma_{tr},\qquad
s_t=\sum_{k=1}^t h_{t,k}^2.
\end{aligned}
\]
Here \(h_{t,k}\) is the stale mass, namely the weight on the initial point and outputs from rounds before \(k\). Put \(F(\Gamma)=T^{-1/2}\sum_{t=1}^T\sqrt{s_t}\). If \(\mathsf V\) is the SDP optimum, the appendix proves
\begin{equation}\label{eq-schedule-bounds}
\mathsf V\ge\max\left\{\frac A{\sqrt2},\ F(\Gamma)\right\}.
\end{equation}
The row norm \(a_t\) measures concentration, while \(h_{t,k}\) measures how much weight remains before cutoff \(k\). Appendix~\ref{app-harmonic} proves the second bound using a harmonic potential, interpreted as the probability that a walk on the graph of oracle-comparison multipliers reaches the comparator before an old output or the origin. This choice cancels the signed query coefficients from the dual inequality.

\subsection{Certified finite-horizon examples}\label{sec-experiments}
For the tuned schedule in Section~\ref{sec-strict}, we compare instances with unique oracle replies against the numerical worst-case regret in \citet[Figure~1, top left]{WeibelEtAl2026}. With $L=D=1$ and $x_1=0$, we use $T=10,20,30,40,50,60$, six equally spaced horizons on their grid. At each horizon, we solve the full PEP, repair its Gram factor to make every nontrivial oracle inequality strict, and verify the resulting instance with exact integer arithmetic and outward intervals for the algebraic coefficients.

Table~\ref{tab:certified-examples} gives the certified regret lower bounds. They lie within $0.1\%$ of the published numerical values at the same horizons. Thus almost all of the regret found by the numerical searches persists with unique oracle replies. The normalized lower bounds are about $1.18$, roughly $1.4$ times the universal path bound, and hold for every exact LMO. Appendix~\ref{app-experiments} plots the comparison with the known upper bound and gives the certification and reproduction details.

\begin{table}[ht]
\centering\small\setlength{\tabcolsep}{3pt}
\caption{Certified regret lower bounds $\underline R_T$ and normalized values for the tuned schedule.}
\label{tab:certified-examples}
\begin{tabular}{r r r}
\toprule
$T$ & $\underline R_T$ & $\underline R_T/T^{3/4}$ (approx.) \\
\midrule
10 & 6.661 & 1.185 \\
20 & 11.19 & 1.183 \\
30 & 15.14 & 1.182 \\
40 & 18.78 & 1.181 \\
50 & 22.21 & 1.181 \\
60 & 25.47 & 1.182 \\
\bottomrule
\end{tabular}

\end{table}

\section{Discussion}\label{sec-discussion}
The path gives one hard instance for every span-query policy, including adaptive and randomized coefficients. Resisting rotations extend its consequence to arbitrary deterministic queries under the oracle-only feasibility requirement. For the specified fixed schedules, strict Gram mixing gives a separate guarantee that is independent of oracle tie-breaking. The remaining questions concern the leading constant and the dimension. How close can the lower constant come to the known upper bound, and what rates are unavoidable when the dimension is fixed independently of the horizon?
\label{last-main-page}
\subsection*{AI Use Statement}
The authors use ChatGPT Astra and Claude Fable for literature survey, language editing and formatting. They also use AI-assisted reviewing to obtain preliminary feedback and improve the paper. All proofs, analyses, and conclusions remain the sole responsibility of the authors

\clearpage
\bibliography{references}
\clearpage
\appendix
\section{Path geometry and gradient norms}\label{app-chain}
This section proves the three relations used in Section~\ref{sec-chain}. The matrix \(\mathbf A_{\rm p}\) is positive definite. For any \(y\in\R^M\), extend the coordinates by \(y_0=y_{M+1}=0\), and observe that
\[
y^\top\mathbf A_{\rm p}y=\sum_{i=0}^M(y_{i+1}-y_i)^2.
\]
Thus \(\mathbf K\) is positive definite, so its increments are linearly independent and the \(M+1\) path points are affinely independent.

If \(h\in\R^M\) is the indicator of a nonempty contiguous interval of length \(\ell\), then \(h^\top\mathbf A_{\rm p}h=2\) and \(h^\top h=\ell\). Since \(w_i-w_j\) is the sum of the increments in the interval \(i,\ldots,j-1\), this proves \eqref{eq-path-distance}. Its maximum \(D^2\) is attained by \(w_1,w_{M+1}\). For any two convex combinations of the vertices, the triangle inequality bounds their distance by the maximum vertex distance, so the convex hull has the same diameter.

For \(2\le i<j\le M+1\), take the indicators of the nested prefixes \(1,\ldots,i-1\) and \(1,\ldots,j-1\). Their bilinear product under \(\mathbf A_{\rm p}\) is \(1\), and their Euclidean inner product is \(i-1\). This proves \eqref{eq-path-prefix}, and the zero vertex has zero inner product with every point.

For the inverse bound, put \(\mu_{\rm p}=2/M\) and define
\[
\varrho_{\rm p}=\frac{2+\mu_{\rm p}-\sqrt{\mu_{\rm p}(4+\mu_{\rm p})}}2\in(0,1).
\]
Fix \(r\in\{1,\ldots,M\}\), let \(e_r\in\R^M\) be the standard coordinate vector, and define a vector \(\psi\in\R^M\) by
\[
\psi_j=\frac{\varrho_{\rm p}^{|j-r|}}{\sqrt{\mu_{\rm p}(4+\mu_{\rm p})}}.
\]
The geometric recurrence gives
\((\mathbf A_{\rm p}+\mu_{\rm p}I_M)\psi\ge e_r\) coordinatewise, because equality holds on the infinite integer path and omitting an endpoint neighbor adds a positive term. The finite matrix has an entrywise nonnegative inverse. Indeed, if a vector \(y\) has a negative minimum, then \((\mathbf A_{\rm p}+\mu_{\rm p}I_M)y\) is negative at a minimizing coordinate. This also covers endpoints by extending them by zero. Hence
\[
[(\mathbf A_{\rm p}+\mu_{\rm p}I_M)^{-1}]_{rr}
\le\frac1{\sqrt{\mu_{\rm p}(4+\mu_{\rm p})}}
\le\frac{\sqrt M}{2\sqrt2}.
\]
If \(U\) is the square matrix with columns \(\Delta_i\), then \(U^\top U=\mathbf K\) and \(\Delta_r^*=U\mathbf K^{-1}e_r\). Its squared norm is \((\mathbf K^{-1})_{rr}\). Multiplying the last bound by \(4/D^2\) proves \eqref{eq-path-inverse}, including \(M=1\).

\section{The deterministic rotation and replay}\label{app-rotation}
We give the full embedding argument for Theorem~\ref{thm-main}. Fix the ambient dimension and the deterministic learner first. Let \(n_{\rm call}=b(T-1)\), so \(M=n_{\rm call}+1\) and \(d=M+n_{\rm call}\). Use the path from Section~\ref{sec-chain} as an abstract instance in \(\R^M\). In this section, \(\bar g_t=c_{\rm p}\Delta_{k_t}^*\) denotes its gradient before embedding.

Maintain a linear isometry \(\mathcal U:\R^M\to\R^d\), an abstract subspace \(\mathcal V\) spanned by all gradients and replies already revealed, and its fixed image \(\mathcal S=\mathcal U\mathcal V\). Also maintain a normal subspace \(\mathcal N\) such that \(\mathcal U(\R^M)\perp\mathcal N\). Initially \(\mathcal V=\mathcal S=\mathcal N=\{0\}\), and any isometry can be used. Revealing a gradient or reply enlarges \(\mathcal V\) and \(\mathcal S\), and every subsequent isometry must agree on \(\mathcal V\).

At a query \(q\), decompose it orthogonally as
\[
q=q_{\mathcal S}+q_{\mathcal N}+r,
\qquad r\perp(\mathcal S+\mathcal N).
\]
If \(r\ne0\), replace the isometry only on \(\mathcal V^\perp\), placing its image in \((\mathcal S+\mathcal N+\operatorname{span}\{r\})^\perp\), and enlarge \(\mathcal N\) by \(r\). To check existence, write \(h=\dim\mathcal V\) and \(n=\dim\mathcal N\) before the call. The target complement has dimension \(d-h-n-1\), whereas the unused abstract space has dimension \(M-h\). Before any call \(n\le n_{\rm call}-1\), so \(d-h-n-1\ge M-h\). An orthonormal-basis extension supplies the required isometry. If \(r=0\), no change is needed.

For every abstract vertex \(w_j\), the updated embedding satisfies
\[
\ip{q}{\mathcal U w_j}=\ip{q_{\mathcal S}}{\mathcal U w_j}.
\]
The inverse image of \(q_{\mathcal S}\) is in \(\mathcal V\), so the query is equivalent on the whole domain to a span query. Return the embedded least-index minimizing vertex and add it to the exposed span. At each new round, reveal \(\mathcal U\bar g_t\) and add \(\bar g_t\) to \(\mathcal V\).

Every partial history is compatible with the current embedded domain and its least-index selector. To see this, each previous query was the sum of a vector in the then-exposed image and a stored normal vector. Subsequent embeddings preserve the former and remain orthogonal to the latter. They preserve every vertex's old query score, including which least index minimizes it. Lemma~\ref{lem-hull} therefore applies throughout this simulation, and every decision is in the observed hull. The frontier proof of Proposition~\ref{prop-span} now shows that each decision has zero current loss and the final vertex has loss \(-c_{\rm p}\).

Let \(\mathcal U_*\) be the final isometry. Fix the domain \(\mathcal U_*\conv\{w_1,\ldots,w_{M+1}\}\), the initial point zero, the least-index selector on those vertices, and the loss gradients \(\mathcal U_*\bar g_1,\ldots,\mathcal U_*\bar g_T\). All previously exposed vectors retain their values, and the score argument above verifies each past oracle answer for this single final selector. Induction over decisions, gradient feedback and calls reproduces the simulation, since the learner is deterministic. These data can therefore be fixed before the actual play, and their regret is \(c_{\rm p}T\). The isometry preserves diameter and gradient norms. When \(n_{\rm call}=0\), there is no rotation step and the same conclusion follows directly.

\section{Unique minimizers for fixed schedules}\label{app-strict}

We prove the unique-minimizer corollary for the fixed schedule
\eqref{eq-model}. First take $L=D=1$ and $x_1=0$, and put $m=T-1$.
The formal vector list is
\[
 (g_1,\ldots,g_T,v_1,\ldots,v_m,u),
\]
where $u$ is a designated comparator. It contains $2T$ vectors. The
played points and queries are fixed linear combinations of this list,
so the explicit-comparator objective, squared distances and oracle
comparisons are linear expressions in its Gram matrix. In particular, mixing two feasible
Gram matrices mixes every comparison gap and the objective with the
same weights.

\paragraph{A strict zero-objective witness.}
We first construct unscaled vectors, denoted by tildes. At round $t$,
choose a unit vector $\widetilde g_t$ orthogonal to all previously
constructed vectors. When $t\le m$, form
\[
 \widetilde q_t=\sum_{s\le t}\eta_{ts}\widetilde g_s
                 +\sum_{s<t}\beta_{ts}\widetilde v_s.
\]
Let $\widetilde p_0=0$ and $\widetilde p_j=\widetilde v_j$ for $1\le j<t$.
Choose a parent index $\pi_t\in\{0,\ldots,t-1\}$ minimizing
$\langle\widetilde q_t,\widetilde p_j\rangle$, and write
\[
 \mathfrak m_t=\ip{\widetilde q_t}{\widetilde p_{\pi_t}}\le0.
\]
The fresh gradient gives a direction that improves the current query without changing previous query values. Contracting the parent toward the origin ensures that every earlier chosen reply remains strictly better than the new point.
Choose $\upsilon_t\in(0,1)$ so that
$(1-\upsilon_t)|\mathfrak m_t|<|\eta_{tt}|/2$. Such a choice exists
because $\eta_{tt}\ne0$. For example, one may take
\[
 \upsilon_t=1-\min\left\{\frac12,
                  \frac{|\eta_{tt}|}{2(1+|\mathfrak m_t|)}\right\}.
\]
Take a fresh unit vector $\xi_t$ orthogonal to every vector constructed
so far, including $\widetilde g_t$, and set
\begin{equation}\label{eq-strict-parent}
 \widetilde v_t=\upsilon_t\widetilde p_{\pi_t}+\xi_t
                    -\operatorname{sign}(\eta_{tt})\widetilde g_t.
\end{equation}
Finally take a unit comparator $\widetilde u$ orthogonal to all these
vectors. There are $T$ fresh gradient directions, $m$ fresh directions
$\xi_t$, and one comparator direction, so the construction uses $2T$
dimensions.

Freshness of $\widetilde g_t$ gives
\[
 \ip{\widetilde q_t}{\widetilde v_t}
       =\upsilon_t\mathfrak m_t-|\eta_{tt}|
       <\mathfrak m_t-\frac{|\eta_{tt}|}{2}<0.
\]
Thus the new reply strictly improves on every old point, including the
origin. It also preserves every previous strict comparison. Indeed,
for $s<t$, the two new directions in \eqref{eq-strict-parent} are
orthogonal to $\widetilde q_s$. If $\widetilde v_s$ is the strict
minimizer over the points already constructed, then
\begin{align*}
 \ip{\widetilde q_s}{\widetilde v_t}
 &=\upsilon_t\ip{\widetilde q_s}{\widetilde p_{\pi_t}}\\
 &\ge\upsilon_t\ip{\widetilde q_s}{\widetilde v_s}
  >\ip{\widetilde q_s}{\widetilde v_s},
\end{align*}
where the last inequality uses its negative value and
$0<\upsilon_t<1$. The comparator has score zero at every issued query.
Induction therefore makes $\widetilde v_t$ the unique minimizer of
$\widetilde q_t$ over the completed vertex set
$\{0,\widetilde v_1,\ldots,\widetilde v_m,\widetilde u\}$.
Only the comparison of a reply with itself is excluded from strictness.

The unscaled outputs follow parent paths in which each edge multiplies
earlier coordinates by a number in $(0,1)$. Each output index contributes
two orthogonal coordinates, one from $\widetilde g_j$ and one from
$\xi_j$, with coefficient magnitude at most one. In the difference of
two outputs, each such coefficient also has magnitude at most one,
because shared coordinates have the same sign. The union of two parent
paths uses at most $m$ output indices. Hence their squared distance is
at most $2m$, and the squared distance of an output to the fresh
comparator is at most $2m+1$.

Multiply every formal vector by
\[
 \rho_{\rm str}=\frac1{2\sqrt{2m+1}}
\]
and let $\mathbf G_{\rm str}$ be their Gram matrix. Its gradients have
norm at most one, its point set has diameter at most $1/2$, and all
strict comparisons remain strict because every query--point inner
product is multiplied by $\rho_{\rm str}^2$. The vectors are linearly
independent because, in descending output order, each output has a fresh
$\xi_t$ coordinate, and the gradient and comparator directions are
independent. Thus $\mathbf G_{\rm str}\succ0$. Moreover, each fresh
gradient is orthogonal to the earlier output hull and to the comparator,
so the explicit-comparator objective of this witness is zero.

\paragraph{Mixing with the chain.}
Run the fixed schedule on the one-call chain construction, with its
specified least-index selector, and let $\mathbf G_{\rm ch}$ be the
Gram matrix of the same formal vector list. Its point diameter is one,
its gradient norms are at most one, and its explicit-comparator
objective is at least $2^{-1/4}T^{3/4}$. For any $\omega\in(0,1)$, set
\[
 \mathbf G_{\rm mix}=(1-\omega)\mathbf G_{\rm ch}
                         +\omega\mathbf G_{\rm str}.
\]
Every non-self comparison gap is strictly positive because the chain
contributes a nonnegative gap and the strict witness contributes a
positive gap. The matrix is positive definite and can be realized in
$\mathbb R^{2T}$. Its gradient norms are at most one and its point
diameter is at most one. Its explicit-comparator objective is at least
\[
 (1-\omega)2^{-1/4}T^{3/4}.
\]
Induction on the calls now realizes this transcript with every exact
LMO on the constructed polytope, since each prescribed reply is the
unique minimizer at the query actually issued. The claim concerns
these issued queries. Other query directions may have several
minimizers.

\paragraph{Exact diameter and scaling.}
The origin is one of the existing points. Let $\chi$ be the largest
squared norm of these points after the mixture is factored. The
diameter constraints give $\chi\le1-3\omega/4<1$, because the chain
point distances are at most one and the strict point distances are
at most $1/2$. Choose a unit vector $\nu$ orthogonal to all $2T$
formal vectors and add
\[
 p_+=\sqrt{1-\chi}\,\nu.
\]
Its squared distance to an old point $p$ is $1-\chi+\|p\|_2^2\le1$,
with equality at a largest-norm point. The enlarged polytope therefore
has diameter exactly one. Every issued query has score zero at $p_+$,
whereas its designated reply has strictly negative score, so padding
preserves uniqueness. The dimension is at most $2T+1$.

For arbitrary $L,D>0$, apply the normalized construction to the
coefficient arrays $((L/D)\eta,\beta,\gamma)$, with the same horizon
and the chosen ambient dimension. Then multiply gradients by $L$
and points by $D$. Each physical query is $D$ times its normalized
query, so all minimizing replies agree and each query--point comparison
is multiplied by $D^2$. Norms and diameter acquire the desired bounds,
while regret is multiplied by $LD$. The nonzero-current-gradient
assumption is unchanged.

Taking $\omega=1/10$ proves the stated coefficient $3/4$, since
\[
 \frac9{10}2^{-1/4}>\frac34
\]
is equivalent to $2<(6/5)^4$. This argument also works at $T=1$, when
there are no calls, the strictness condition is vacuous, and the
witness has only its gradient and comparator vectors. More generally,
any lower constant strictly below $2^{-1/4}$ follows by choosing a
sufficiently small positive $\omega$.

\section{Canonicalization and the exact Gram problem}
\label{app-canonicalization}

We prove the schedule-specific bounds from Section~\ref{sec-bounds}, with initial point $x_1=0$ and normalized bounds $L=D=1$.

\subsection{Removing redundant calls}

An oracle $\mathcal O$ is \emph{invariant under positive rescaling} if
\[
\mathcal O(cq)=\mathcal O(q)\qquad(c>0).
\]
The adversary may choose an exact oracle whose selection has this property.
Before choosing an instance, a formal query is its coefficient vector over the gradient and output symbols. Queries on the same positive ray must receive the same reply from this oracle, so we merge their calls and add their decision weights.

\begin{lemma}[Canonicalization]
\label{app-lem-canonicalization}
Fix a deterministic scheme with the coefficient restrictions in \eqref{eq-model}. There is a reduced scheme with at most one retained call per round and fixed coefficients satisfying the following properties.
\begin{enumerate}
\item the original and reduced schemes have identical played points against every exact oracle invariant under positive rescaling.
\item the reduced iterate coefficients are nonnegative and their row sums are at most one.
\item at most one retained query is formally zero, and no two retained nonzero queries are formally positive multiples.
\end{enumerate}
\end{lemma}

\begin{proof}
After discarding the call after the last loss, process the remaining calls chronologically. In each direction, substitute the retained aliases of previously removed outputs, then express the result as a coefficient vector in the gradient and retained output symbols.

If this vector is zero and an earlier zero query has been retained, remove the call and alias its output to that earlier output. If it is a positive multiple of an earlier retained nonzero coefficient vector, remove it and use that earlier output as its alias. Otherwise retain the call and introduce its output as a new formal symbol. Coefficient vectors may be padded with zeros when new symbols are introduced.

All decisions and substitutions depend only on the coefficient arrays. Against an oracle invariant under positive rescaling, each removed call receives the answer prescribed by its alias, so induction over calls gives identical original and reduced transcripts. Substitution adds the iterate weights assigned to aliased outputs, preserving nonnegativity and total mass. Causality is preserved because every alias is an earlier retained output, and the retained queries have the asserted formal properties by construction.
\end{proof}

Let $m$ be the number of retained calls and let $\tau_r$ be the round of retained call $r$, where $r\in\{1,\ldots,m\}$. In the one-call model,
\[
\tau_1<\cdots<\tau_m\le T-1,\qquad m\le T-1.
\]
Write $v_r$ for retained output $r$. The reduced coefficient matrices are
\[
E\in\mathbb R^{m\times T},\qquad
B\in\mathbb R^{m\times m},\qquad
\Gamma\in\mathbb R^{T\times m}.
\]
They satisfy
\[
E_{rt}=0\ \text{if }t>\tau_r,\qquad
B_{rj}=0\ \text{if }j\ge r,\qquad
\Gamma_{tr}=0\ \text{if }\tau_r\ge t,
\]
and $\Gamma_{tr}\ge0$, $\sum_r\Gamma_{tr}\le1$. Thus
\begin{equation}
q_r=\sum_{t=1}^T E_{rt}g_t+\sum_{j=1}^m B_{rj}v_j,
\qquad
x_t=\sum_{r=1}^m\Gamma_{tr}v_r.
\label{app-eq-reduced-model}
\end{equation}
Empty matrices and sums have their usual meanings when $m=0$.

\subsection{Primal and dual formulations}
\label{app-pep-duality}

Define the point index set
\[
J=\{0,1,\ldots,m,\star\},
\qquad
p_0=0,\quad p_r=v_r\ (1\le r\le m),\quad p_\star=u,
\]
where $u$ is a comparison point. Order these indices as $0<1<\cdots<m<\star$. For $j\in J\setminus\{0\}$, let $e_j$ denote the corresponding standard basis vector of $\mathbb R^{m+1}$, and set $e_0=0$. Let
\[
N=T+m+1,\qquad
\bar B=[B,0]\in\mathbb R^{m\times(m+1)},\qquad
H=[\Gamma,-\mathbf 1]\in\mathbb R^{T\times(m+1)},
\]
where $\mathbf 1$ is the length-$T$ vector of ones. Define the embedding $\iota:\mathbb R^{m+1}\to\mathbb R^N$ by $\iota(z)=(0,z)$, with $T$ initial zero coordinates. The formal coefficient vector of retained query $r$ is
\[
\vartheta_r=(E_{r,:}^{\mathsf T},\bar B_{r,:}^{\mathsf T})\in\mathbb R^N.
\]

Let $\mathbb S^N$ denote the real symmetric $N\times N$ matrices. A Gram matrix $\mathbf G\in\mathbb S^N$ is ordered by the vectors
\[
(g_1,\ldots,g_T,v_1,\ldots,v_m,u)
\]
and partitioned into blocks $\mathbf G_{gg}$, $\mathbf G_{gp}$, and $\mathbf G_{pp}$, with gradient block size $T$ and point block size $m+1$. For matrices of equal size, $\langle\cdot,\cdot\rangle_F$ denotes the Frobenius inner product. Define the objective
\[
\Phi(\mathbf G)=\langle H,\mathbf G_{gp}\rangle_F.
\]
The transcript performance-estimation problem is
\begin{equation}
\begin{aligned}
\mathsf V=\max_{\mathbf G\in\mathbb S^N}\quad
 &\Phi(\mathbf G)\\
\text{subject to}\quad
 &\mathbf G\succeq0,\\
 &(\mathbf G_{gg})_{tt}\le1 &&(1\le t\le T),\\
 &(e_i-e_j)^{\mathsf T}\mathbf G_{pp}(e_i-e_j)\le1
 &&(i,j\in J,\ i<j),\\
 &\vartheta_r^{\mathsf T}\mathbf G\,\iota(e_r-e_j)\le0
 &&(1\le r\le m,\ j\in J\setminus\{r\}).
\end{aligned}
\label{app-eq-primal-full}
\end{equation}
The last constraints impose every oracle comparison, including comparisons with future outputs, the origin, and the comparison point.

To specify the dual, let $\lambda_t\ge0$ multiply the gradient constraint at round $t$, let $\mu_{ij}\ge0$ multiply the distance constraint for $i<j$, and let $\alpha_{rj}\ge0$ multiply the oracle constraint for call $r$ and point $j\ne r$. Set $\alpha_{rr}=0$. Define
\[
\Lambda=\sum_{t=1}^T\lambda_t,\qquad
\Upsilon=\sum_{\substack{i,j\in J\\i<j}}\mu_{ij},
\]
\[
\mathbf Q_{r,:}=\sum_{j\in J}\alpha_{rj}(e_r-e_j)^{\mathsf T},
\qquad
\mathbf P_\mu=\sum_{\substack{i,j\in J\\i<j}}
\mu_{ij}(e_i-e_j)(e_i-e_j)^{\mathsf T}.
\]
Thus $\mathbf Q\in\mathbb R^{m\times(m+1)}$. For a square matrix $W$, define $\operatorname{sym}(W)=(W+W^{\mathsf T})/2$. Write $\lambda=(\lambda_1,\ldots,\lambda_T)^{\mathsf T}$ and let $\operatorname{diag}(\lambda)$ denote the diagonal matrix with these entries. The dual minimizes $\Lambda+\Upsilon$ subject to the nonnegative multipliers and the following matrix constraint.
\begin{equation}
\mathbf M=
\begin{pmatrix}
\operatorname{diag}(\lambda)&(E^{\mathsf T}\mathbf Q-H)/2\\
(\mathbf Q^{\mathsf T}E-H^{\mathsf T})/2&
\mathbf P_\mu+\operatorname{sym}(\bar B^{\mathsf T}\mathbf Q)
\end{pmatrix},
\qquad \mathbf M\succeq0.
\label{app-eq-dual-full}
\end{equation}
We use $\mathbf M_{pp}$ for its lower-right point block.

\begin{lemma}[Gram realization and strong duality]
\label{app-lem-duality}
Problem~\eqref{app-eq-primal-full} has a finite attained optimum. Every feasible matrix admits a vector realization in dimension at most $N$ satisfying the gradient, diameter and oracle comparisons, with comparison objective $\Phi(\mathbf G)$. Lemma~\ref{app-lem-fixed-oracle} realizes an arbitrarily large fraction of the optimum with a fixed selector. The primal and dual have no duality gap.
\end{lemma}
\begin{proof}[Proof of Lemma~\ref{app-lem-duality}]
The zero matrix is primal feasible. Distance constraints involving the origin bound every point diagonal entry by one, and gradient constraints do the same for gradient diagonal entries. Positive semidefiniteness then bounds every entry in absolute value by one. The feasible set is closed and bounded, hence compact, so the optimum is finite and attained.

Factor a feasible matrix as $\mathbf G=Z^{\mathsf T}Z$ and use the columns of $Z$ as the indicated vectors. The constraints give their norm and distance bounds and the oracle inequalities. The domain $\mathcal C=\operatorname{conv}\{p_j:j\in J\}$ has diameter at most one, since the distance between two convex combinations is at most the largest distance between vertices. Each indicated response minimizes its linear functional over every vertex and hence over $\mathcal C$. The objective is
\[
\Phi(\mathbf G)=\sum_t\langle g_t,x_t-u\rangle,
\]
which does not exceed regret because $u\in\mathcal C$. Conversely, the Gram data of any admissible transcript and any feasible comparison point satisfy the program. A possible conflict between repeated queries and a fixed oracle selection is handled separately in Section~\ref{app-fixed-oracle}.

The Lagrangian has the form
\[
\Lambda+\Upsilon-\langle\mathbf M,\mathbf G\rangle_F.
\]
The oracle contribution to its negative Gram coefficient is the sum of symmetric products of $\vartheta_r$ and $\iota(e_r-e_j)$ with weights $\alpha_{rj}$. Its gradient--point block is $E^{\mathsf T}\mathbf Q/2$, and its point block is $\operatorname{sym}(\bar B^{\mathsf T}\mathbf Q)$. Together with the objective contribution $-H/2$ in the gradient--point block, these give~\eqref{app-eq-dual-full}. The supremum of the Lagrangian over $\mathbf G\succeq0$ is finite precisely when $\mathbf M\succeq0$.

To obtain a strictly feasible dual point, choose every scalar multiplier positive, then add a common amount $\zeta>0$ to each $\lambda_t$ and each origin-distance multiplier $\mu_{0j}$, $j\ne0$. This adds $\zeta$ times the identity matrix to $\mathbf M$, making it positive definite for sufficiently large finite $\zeta$ while all scalar multipliers remain positive. Semidefinite strong duality \citep[Section~5.9.1]{BoydVandenberghe2004} therefore equates the primal maximum with the dual infimum, whether or not the latter is attained.
\end{proof}

\subsection{Concentrated weights}
Define the row norms and their sum by
\[
a_t=\left(\sum_{r=1}^m\Gamma_{tr}^2\right)^{1/2},
\qquad A=\sum_{t=1}^T a_t.
\]
\begin{proposition}\label{prop-concentration}
For \(T\ge2\), the optimum in \eqref{app-eq-primal-full} satisfies \(\mathsf V\ge A/\sqrt2\).
\end{proposition}
\begin{proof}
Let \(\varepsilon_1,\ldots,\varepsilon_T\) be orthonormal vectors and take
\(\mathcal C=\conv\{0,\varepsilon_1/\sqrt2,\ldots,\varepsilon_T/\sqrt2\}\).
Let the oracle select the smallest-index minimizing nonzero vertex whenever such a vertex is at least as good as zero, and return zero otherwise. Recursively set \(g_t=x_t/\norm{x_t}\) when \(x_t\ne0\), and set \(g_t=0\) otherwise.

Every query is supported on previously returned coordinates, irrespective of the signs in \(E,B\), while at least one unvisited coordinate has query value zero. If \(i_r\) is the coordinate index of the returned nonzero vertex at call \(r\), then
\[
\norm{x_t}^2
=\frac12\sum_{i=1}^T\left(\sum_{r:i_r=i}\Gamma_{tr}\right)^2
\ge\frac12 a_t^2.
\]
The inequality uses \(\Gamma_{tr}\ge0\). The gradients are coordinatewise nonnegative, so zero is a best comparator and regret equals \(\sum_t\norm{x_t}\). The recursive sequence can be computed before play and replayed unchanged.
\end{proof}

\section{The harmonic-potential argument}
\label{app-harmonic}

We lower-bound every feasible dual objective, so strong duality gives the same lower bound on \(\mathsf V\). Fix a feasible dual point of~\eqref{app-eq-dual-full} and regard its oracle multipliers \(\alpha_{rj}\) as directed edge weights from call \(r\) to comparison point \(j\). A walk follows an outgoing edge with probability proportional to its weight. For a cutoff round \(k\), give a starting point the probability that this walk reaches the comparator before the origin or an output from before \(k\). Paths that never reach these boundary points contribute zero, and a state with no outgoing weight stays put. The resulting potential is zero before the cutoff and harmonic afterwards, meaning that its value equals the weighted average of the next values. These properties cancel the signed query terms.

\begin{lemma}[Existence of a bounded harmonic potential]
\label{app-lem-harmonic-existence}
For every round $k\in\{1,\ldots,T\}$ there is a vector $z^{(k)}\in\mathbb R^{m+1}$ such that, with the convention $z^{(k)}_0=0$,
\[
z^{(k)}_\star=1,\qquad
z^{(k)}_r=0\ \text{if }\tau_r<k,\qquad
0\le z^{(k)}_r\le1,
\]
and
\[
(\mathbf Qz^{(k)})_r=0\qquad\text{if }\tau_r\ge k.
\]
\end{lemma}

\begin{proof}
Let $\mathcal R_k=\{r:\tau_r\ge k\}$ be the set of future call indices. Keep the prescribed boundary coordinates fixed and initialize the coordinates in $\mathcal R_k$ to zero. For each $r\in\mathcal R_k$, define the total outgoing weight
\[
\kappa_r=\sum_{j\in J}\alpha_{rj}.
\]
Let \(n\ge0\) denote the iteration index and write \(z^{(k,n)}\) for the iterate associated with cut \(k\). If \(\kappa_r>0\), iterate
\[
z_r^{(k,n+1)}=
\frac{\alpha_{r\star}+\sum_{j\in\mathcal R_k}\alpha_{rj}z_j^{(k,n)}}
{\kappa_r}.
\]
If \(\kappa_r=0\), keep $z_r^{(k,n)}=0$. The nonnegative update coefficients make the sequence coordinatewise nondecreasing, while the numerator is at most the denominator whenever the coordinates are bounded by one. Thus every coordinate converges to a limit \(z^{(k)}\in[0,1]^{m+1}\), and passing to the limit in the finite sums gives the harmonic equations. Iteration from zero gives the minimal nonnegative solution. Each iterate counts successful absorption within a finite number of steps, so its limit has the probability interpretation above, including on closed classes that never reach the boundary.
\end{proof}

For $1\le k\le t\le T$, define the stale mass
\[
h_{t,k}=1-\sum_{r:k\le\tau_r<t}\Gamma_{tr}.
\]
This is the mass on the initial point and on outputs from rounds before $k$, so it belongs to $[0,1]$. Define the sum of its squares over cutoffs by
\[
s_t=\sum_{k=1}^t h_{t,k}^2.
\]

\begin{proposition}[Harmonic bound]\label{prop-harmonic}
\label{app-lem-cut}
Every feasible dual point has $\lambda_t>0$ for all $t$ and satisfies, for every round $k$,
\begin{equation}
\Upsilon\ge\frac14\sum_{t=k}^T\frac{h_{t,k}^2}{\lambda_t}.
\label{app-eq-cut}
\end{equation}
Consequently,
\begin{equation}
\mathsf V\ge\frac1{\sqrt T}\sum_{t=1}^T\sqrt{s_t}.
\label{app-eq-temporal-bound}
\end{equation}
\end{proposition}

\begin{proof}
Fix $k$ and abbreviate the potential in Lemma~\ref{app-lem-harmonic-existence} by $z$. Causality removes contributions from calls before the cutoff, while harmonicity removes those from later calls. Together they cancel both signed coefficient matrices \(E\) and \(B\) without discarding any oracle comparison. If $\tau_r<k$, all outputs available in query $r$ have potential zero, so $(\bar Bz)_r=0$. On the complementary indices, $(\mathbf Qz)_r=0$, giving
\begin{equation}
z^{\mathsf T}\operatorname{sym}(\bar B^{\mathsf T}\mathbf Q)z
=\langle\bar Bz,\mathbf Qz\rangle=0.
\label{app-eq-beta-cancellation}
\end{equation}
For a gradient index $t\ge k$, the coefficient $E_{rt}$ can be nonzero only if $\tau_r\ge t\ge k$, so $(E^{\mathsf T}\mathbf Qz)_t=0$ and
\begin{equation}
[(E^{\mathsf T}\mathbf Q-H)z]_t
=1-\sum_{r:k\le\tau_r<t}\Gamma_{tr}z_r
\ge h_{t,k}\qquad(t\ge k).
\label{app-eq-eta-cancellation}
\end{equation}
The last inequality uses nonnegativity of $\Gamma$. By~\eqref{app-eq-beta-cancellation},
\begin{equation}
z^{\mathsf T}\mathbf M_{pp}z
=\sum_{i<j}\mu_{ij}(z_i-z_j)^2\le\Upsilon,
\label{app-eq-potential-cost}
\end{equation}
since all point potentials are in $[0,1]$.

First take $k=t$. Then $h_{t,t}=1$, so~\eqref{app-eq-eta-cancellation} gives a nonzero cross term between gradient coordinate $t$ and point vector $z$. A positive-semidefinite quadratic form with zero diagonal coefficient at that gradient coordinate could not have such a cross term. Thus $\lambda_t>0$.

For general $k$, evaluate the quadratic form of $\mathbf M$ on vectors $(y,z)$, where $y\in\mathbb R^T$ is supported on indices $t\ge k$. Its minimum over these gradient coordinates is
\[
z^{\mathsf T}\mathbf M_{pp}z
-\frac14\sum_{t=k}^T
\frac{([(E^{\mathsf T}\mathbf Q-H)z]_t)^2}{\lambda_t}.
\]
It is nonnegative because $\mathbf M\succeq0$. Combining this fact with~\eqref{app-eq-eta-cancellation}--\eqref{app-eq-potential-cost} proves~\eqref{app-eq-cut}.

Average~\eqref{app-eq-cut} over $k=1,\ldots,T$. Cauchy--Schwarz gives
\[
\Upsilon\ge\frac1{4T}\sum_{t=1}^T\frac{s_t}{\lambda_t}
\ge\frac{(\sum_t\sqrt{s_t})^2}{4T\Lambda}.
\]
Consequently $\Lambda+\Upsilon\ge T^{-1/2}\sum_t\sqrt{s_t}$. This bound holds for every feasible dual point, so strong duality proves~\eqref{app-eq-temporal-bound}.
\end{proof}

\section{Fixed-oracle realization of the SDP bounds}
\label{app-fixed-oracle}

To realize a Gram solution with one oracle, we need distinct retained query rays to remain distinct after factorization. A positive-definite perturbation ensures this while preserving an arbitrarily large fraction of the objective.

\begin{lemma}[A full-rank feasible witness of zero objective]
\label{app-lem-pd-witness}
For every reduced coefficient array, the feasible set of~\eqref{app-eq-primal-full} contains a positive-definite matrix $\mathbf G_{\mathrm{pd}}$ with $\Phi(\mathbf G_{\mathrm{pd}})=0$.
\end{lemma}

\begin{proof}
Use superscript $(0)$ for the auxiliary vectors in this construction, and process rounds chronologically. At round $t$, choose $g_t^{(0)}$ as a fresh unit vector orthogonal to all previously constructed vectors. At a retained call $r$ in that round, compute $q_r^{(0)}$ from~\eqref{app-eq-reduced-model} and choose a parent index
\[
\pi_r\in\arg\min_{j\in\{0,1,\ldots,r-1\}}
\langle q_r^{(0)},p_j^{(0)}\rangle,
\qquad p_0^{(0)}=0,
\]
where $p_j^{(0)}=v_j^{(0)}$ for $j>0$. Choose a fresh unit vector $\xi_r$ orthogonal to every vector constructed so far, and set
\[
v_r^{(0)}=p_{\pi_r}^{(0)}+\xi_r.
\]
Every earlier query and the current query lie in the span to which $\xi_r$ is orthogonal, so the new output and its parent have identical values under these queries. By induction, the parent satisfies every earlier oracle supporting inequality, which the new vertex therefore preserves. Its current value is no greater than the values at zero and at any previous output. This proves all oracle comparisons among the origin and output vectors, including comparisons with outputs constructed later.

After the last round, choose $u^{(0)}$ as one more fresh unit vector orthogonal to every vector already constructed. Its value under each query is zero, whereas each prescribed output has query value at most zero, so all oracle comparisons with $u^{(0)}$ hold as well.

Each new vector has a nonzero coordinate orthogonal to all preceding vectors, giving a linearly independent list in dimension $N$. Each output is a sum of the unit edge vectors on a path in the parent tree, so its squared norm is at most $m$, squared distances between outputs are at most $m$, and squared distances from an output to $u^{(0)}$ are at most $m+1$. Multiply every vector by
\[
\rho=\frac1{2\sqrt{m+1}}.
\]
The gradient norms and point diameter are then at most $1/2$. Because all queries scale by the same factor, every oracle inequality is preserved, and the resulting Gram matrix $\mathbf G_{\mathrm{pd}}$ is feasible and positive definite.

Finally, each $g_t^{(0)}$ is orthogonal to all outputs available in $x_t$, and $u^{(0)}$ is orthogonal to every gradient. These properties survive uniform scaling and give $\Phi(\mathbf G_{\mathrm{pd}})=0$.
\end{proof}

\begin{lemma}[Fixed-oracle realization]
\label{app-lem-fixed-oracle}
For every $\omega\in(0,1)$, the reduced scheme has a domain of diameter at most one in dimension at most $N$, a fixed deterministic vertex-valued exact oracle invariant under positive rescaling, and a fixed gradient sequence with regret at least $(1-\omega)\mathsf V$. The original scheme has the same transcript after reinstating aliased calls.
\end{lemma}

\begin{proof}
Let $\mathbf G_*$ attain~\eqref{app-eq-primal-full} and define
\[
\widehat{\mathbf G}=(1-\omega)\mathbf G_*+\omega\mathbf G_{\mathrm{pd}}.
\]
It is feasible and positive definite, with objective $(1-\omega)\mathsf V$. Factor it as $Z^{\mathsf T}Z$ and take the columns as the gradient, output, and comparison vectors. Positive definiteness of $\widehat{\mathbf G}$ makes the map $a\mapsto Za$ injective, so it preserves and reflects linear relations. By canonicalization, at most one retained query is zero and no two retained nonzero actual queries lie on the same positive ray.

Set $\mathcal C=\operatorname{conv}\{0,v_1,\ldots,v_m,u\}$. The point vectors $v_1,\ldots,v_m,u$ are linearly independent, so these points and zero are vertices of a simplex. Prescribe $\mathcal O(cq_r)=v_r$ for every retained nonzero query and every $c>0$, and prescribe the response at zero if a retained zero query exists. These prescriptions are consistent, and the primal oracle inequalities make each response a minimizing vertex. At all other directions, return the first minimizing vertex in a fixed ordering of the domain vertices. The resulting oracle is deterministic, vertex-valued, exact, and invariant under positive rescaling.

With all gradients chosen by the factorization before play, induction over rounds and calls shows that the original scheme produces the factored transcript, as retained calls receive their prescribed outputs and removed calls return their retained aliases. The comparison point $u$ is feasible, so regret is at least $\Phi(\widehat{\mathbf G})=(1-\omega)\mathsf V$.
\end{proof}

\section{Additional experimental details}\label{app-experiments}
\paragraph{Source and model match.} The six horizons in Table~\ref{tab:certified-examples} match entries in the saved tuned-schedule curve supplied with \citet[Figure~1, top left]{WeibelEtAl2026}. The schedule, gradient norm and diameter bounds, initial point up to translation, and regret objective agree with ours. We omit only the final oracle call, which has no effect on regret. Figure~\ref{fig:computational-checks} compares these instances with that numerical curve and the analytical bounds. The reported agreement is the absolute difference divided by the saved numerical value. A certified lower bound can slightly exceed this floating-point estimate, so the comparison does not certify an optimality gap or an asymptotic leading constant.

\paragraph{Exact certification.} The full PEP enforces every gradient norm bound, every pairwise distance bound among the origin, oracle outputs and comparator, and every oracle comparison against these points, including future outputs. We round a Euclidean factor, add a small explicitly factored feasible Gram matrix to make nontrivial oracle comparisons strict, and apply a common scaling to restore the norm and diameter bounds. The resulting rational Gram matrix is positive semidefinite by its explicit factors. An independent verifier reconstructs it with integer arithmetic and checks all norm and distance inequalities exactly. It encloses the exact algebraic parameters $\theta^4=27/(16T^3)$ and $\sigma^2=3/T$ in rational intervals and uses outward Horner interval evaluation for the geometric weights. Every nontrivial oracle comparison has a strictly positive lower bound, including comparisons with the origin, past and future outputs, and the comparator. Thus every issued query has a unique minimizer on the convex hull of these points. The same interval calculation bounds regret from below, and the verifier checks the claimed decimal bound without a numerical tolerance. The table rounds certified lower bounds downward.

\paragraph{Numerical searches.} The six searches use Clarabel and all report an optimal status. Strict feasibility repair changes each numerical objective by less than $0.01\%$.

\begin{figure}[H]
\centering
\includegraphics[width=5.5in]{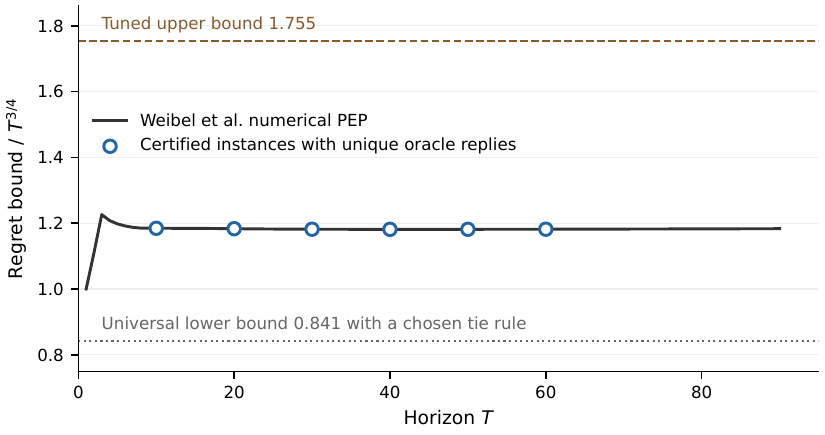}
\caption{Comparison for the tuned schedule with $L=D=1$. Certified instances with unique oracle replies are plotted against the numerical PEP curve of \citet[Figure~1, top left]{WeibelEtAl2026}. The reference lines give their Theorem~3.1 upper bound and our Theorem~\ref{thm-main} lower bound for $b=1$ with a chosen tie rule.}
\label{fig:computational-checks}
\end{figure}

\end{document}